\documentclass[11pt]{article}

\usepackage[utf8]{inputenc} 
\usepackage[T1]{fontenc}    
\usepackage{hyperref}       
\usepackage{url}            
\usepackage{booktabs}       
\usepackage{amsfonts}       
\usepackage{nicefrac}       
\usepackage{microtype}      
\usepackage{enumitem}
\usepackage{graphicx}
\graphicspath{{./images/}}
\usepackage[titletoc,title]{appendix}
\usepackage{amsmath,amsthm,amssymb,mathtools}
\usepackage{braket}

\usepackage{tikz}
\usepackage{multirow}
\usepackage[noend]{algpseudocode}

\usepackage{fullpage}

\DeclarePairedDelimiter\inner{\langle}{\rangle}
\DeclarePairedDelimiter\abs{\lvert}{\rvert}

\newtheorem{theorem}{Theorem}[section]
\newtheorem*{theorem*}{Theorem}

\newtheorem{proposition}[theorem]{Proposition}
\newtheorem*{proposition*}{Proposition}
\newtheorem{lemma}[theorem]{Lemma}
\newtheorem*{lemma*}{Lemma}

\newtheorem*{fact*}{Fact}

\newtheorem*{hypothesis*}{Hypothesis}

\newtheorem*{conjecture*}{Conjecture}
\newtheorem{itheorem}[theorem]{Informal Theorem}

\newtheorem*{claim*}{Claim}

\theoremstyle{definition}
\newtheorem{definition}[theorem]{Definition}

\newtheorem{algorithm}[theorem]{Algorithm}
\newtheorem*{algorithm*}{Algorithm}

\theoremstyle{remark}

\newtheorem*{remark*}{Remark}

\newtheorem*{observation*}{Observation}

\newcommand{\R}{\mathbb{R}}

\newcommand{\Z}{\mathbb{Z}}

\newcommand{\poly}{\mathrm{poly}}

\newcommand{\norm}[1]{\lVert #1 \rVert}
\newcommand{\Norm}[1]{\left\lVert#1\right\rVert}

\newcommand{\Bignorm}[1]{\Big\lVert#1\Big\rVert}

\newcommand{\Iprod}[1]{\left\langle#1\right\rangle}

\DeclareMathOperator{\diam}{diam}

\newcommand{\eps}{\varepsilon}
\renewcommand{\epsilon}{\varepsilon}

\newcommand{\wmin}{w_{\text{min}}}

\newcommand{\APSM}{(\rho,\Delta,\eps)\text{-separated}}
\newcommand{\nice}[1]{S_{#1}^{(\text{nice})}}
\newcommand{\cone}[1]{S_{#1}^{(\text{cone})}}
\newcommand{\good}[1]{S_{#1}^{(\text{good})}}
\DeclareMathOperator{\sgn}{sgn}
\newcommand{\core}[1]{S_{#1}^{(\text{core})}}
\newcommand{\enice}[1]{S_{#1}^{(\text{e nice})}}
\newcommand{\renice}[1]{S_{#1}^{(\text{r e nice})}}
\newcommand{\rgood}[1]{S_{#1}^{(\text{r good})}}

\title{Clustering Stable Instances of Euclidean $k$-means}

\author{
    Abhratanu Dutta$^{*}$ \and
    Aravindan Vijayaraghavan\thanks{Supported by the National Science Foundation (NSF) under Grant No. CCF-1652491 and CCF-1637585.} \and
    Alex Wang
}

\begin{document}
\maketitle

\thispagestyle{empty}

\abstract{
The Euclidean $k$-means problem is arguably the most widely-studied clustering problem in machine learning. While the $k$-means objective is NP-hard in the worst-case, practitioners have enjoyed remarkable success in applying heuristics like Lloyd's algorithm for this problem. To address this disconnect, we study the following question: {\em what properties of real-world instances will enable us to design efficient algorithms and prove guarantees for finding the optimal clustering?} 
We consider a natural notion called additive perturbation stability that we believe captures many practical instances. Stable instances have unique optimal $k$-means solutions that do not change even when each point is perturbed a little (in Euclidean distance). This captures the property that the $k$-means optimal solution should be tolerant to measurement errors and uncertainty in the points. 
We design efficient algorithms that provably recover the optimal clustering for instances that are additive perturbation stable. When the instance has some additional separation, we show an efficient algorithm with provable guarantees that is also robust to outliers.
We complement these results by studying the amount of stability in real datasets and demonstrating that our algorithm performs well on these benchmark datasets.
}
\newpage
\setcounter{page}{1}


\section{Introduction} \label{sec:intro}
One of the major challenges in the theory of clustering is to bridge the large disconnect between our theoretical and practical understanding of the complexity of clustering. While theory tells us that most common clustering objectives like $k$-means or $k$-median clustering problems are intractable in the worst case, many heuristics like Lloyd's algorithm or k-means++ seem to be effective in practice. In fact, this has led to the ``CDNM'' thesis~\cite{BL,CDNMthesis}: ``Clustering is difficult only when it does not matter''.

We try to address the following natural questions in this paper: {\em Why are real-world instances of clustering easy? Can we identify properties of real-world instances that make them tractable?}

We focus on the Euclidean $k$-means clustering problem where we are given $n$ points $X=\set{x_1,\dots,x_n} \subset \R^d$, and we need to find $k$ centers $\mu_1,\mu_2,\dots,\mu_k \in \R^d$ minimizing the objective $\sum_{x \in X} \min_{i\in[k]}\Norm{x-\mu_i}^2$. The $k$-means clustering problem is the most well-studied objective for clustering points in Euclidean space~\cite{AV07}. The problem is NP-hard in the worst-case~\cite{dasgupta2008hardness} even for $k=2$, and a constant factor hardness of approximation is known for larger $k$ ~\cite{ACKS14}.

One way to model real-world instances of clustering problems is through {\em instance stability}, which is an implicit structural assumption about the instance. Practically interesting instances of the $k$-means clustering problem often have a clear optimal clustering solution (usually the ground-truth clustering) that is stable: i.e., it remains optimal even under small perturbations. As argued in \cite{BBG}, clustering objectives like $k$-means are often just a proxy for recovering a ground-truth clustering that is close to the optimal solution. Instances in practice always have measurement errors, and optimizing the $k$-means objective is meaningful only when the optimal solution is stable to these perturbations.

This notion of stability was formalized independently in a pair of influential works ~\cite{BL,BBG}. The predominant strand of work on instance stability assumes that the optimal solution is resilient to multiplicative perturbations of the distances ~\cite{BL}. For any $\gamma \ge 1$, a metric clustering instance $(X,d)$ on point set $X \subset \R^d$ and metric $d: X \times X \to \R_+$ is said to be $\gamma$-factor stable iff the (unique) optimal clustering $C_1, \dots, C_k$ of $X$ remains the optimal solution for any instance $(X, d')$ where any (subset) of the the distances are increased by up to a $\gamma$ factor i.e., $d(x,y) \le d'(x,y) \le \gamma d(x,y)$ for any $x,y \in X$. 
In a series of recent works \cite{ABS12,BalcanLiang} culminating in \cite{AMM17}, it was shown that $2$-factor perturbation stable (i.e., $\gamma\ge 2$) instances of $k$-means can be solved in polynomial time.

Multiplicative perturbation stability represents an elegant, well-motivated formalism that captures robustness to measurement errors for clustering problems in general metric spaces ($\gamma=1.1$ captures relative errors of 10\% in the distances). However, multiplicative perturbation stability has the following drawbacks in the case of Euclidean clustering problems:

\begin{itemize}
	\item Measurement errors in Euclidean instances are better captured using additive perturbations. Uncertainty of $\delta$ in the position of $x, y$ leads to an additive error of $\delta$ in $\norm{x-y}_2$, irrespective of how large or small $\norm{x-y}_2$ is.
	\item The amount of stability, $\gamma$, needed to enable efficient algorithms (i.e., $\gamma \ge 2$) often imply strong structural conditions, that are unlikely to be satisfied by many real-world datasets. For instance, $\gamma$-factor perturbation stability implies that every point is a multiplicative factor of $\gamma$ closer to its own center than to any other cluster center.
	\item Algorithms that are known to have provable guarantees under multiplicative perturbation stability are based on single-linkage or MST algorithms that are very non-robust by nature. In the presence of a few outliers or noise, any incorrect decision in the lower layers gets propagated up to the higher levels.
\end{itemize}

In this work, we consider a natural additive notion of stability for Euclidean instances: the optimal clustering should not change even when each point is moved a Euclidean distance of at most $\delta$. This corresponds to a small additive perturbation to the pairwise distances between the points\footnote{Note that not all additive perturbations to the distances can be captured by an appropriate movement of the points in the cluster. Hence the notion we consider in our paper is a weaker assumption on the instance.}. Unlike multiplicative notions of perturbation stability \cite{BL,ABS12}, this notion of additive perturbation is not scale invariant. Hence the normalization or scale of the perturbation is important.

Ackerman and Ben-David~\cite{BenDavid} initiated the study of additive perturbation stability when the distance between any pair of points can be changed by at most $\delta=\eps \diam(X)$ with $\diam(X)$ being the {\em diameter} of the whole dataset. The algorithms take time $n^{O(k/\eps^2)}=n^{O(k\diam^2(X)/\delta^2)}$ and correspond to polynomial time algorithms when $k,1/\eps$ are constants. However, this dependence of $k\diam^2(X)/\delta^2$ in the exponent is not desirable since the diameter is a very non-robust quantity -- the presence of one outlier (that is even far away from the decision boundary) can increase the diameter arbitrarily. Hence, these guarantees are useful mainly when the whole instance lies within a small ball and the number of clusters is small~\cite{BenDavid,Reyzin}. 
Our notion of additive perturbation stability will use a different scale parameter that is closely related to the distance between the centers instead of the diameter $\diam(X)$. Our results for additive perturbation stability have no explicit dependence on the diameter, and allows instances to have potentially unbounded clusters (as in the case of far-way outliers). With some additional assumptions, we also obtain polynomial time algorithmic guarantees for large $k$.

\subsection{Additive perturbation stability and our contributions}

We consider a notion of additive stability where the points in the instance can be moved by at most $\delta=\epsilon D$, where $\eps \in (0,1)$ is a parameter, and $D=\max_{i \ne j} D_{ij}= \max_{i \ne j} \norm{\mu_i - \mu_j}_2$ is the maximum distance between pairs of means. 
Suppose $X$ is a $k$-means clustering instance with optimal clustering $C_1,C_2,\dots,C_k$.
We say that $X$ is $\epsilon$-additive perturbation stable ($\epsilon$-APS) iff every $\delta$-additive perturbation of $X$ has $C_1,C_2,\dots,C_k$ as an optimal clustering solution.
Note that there is no restriction on the diameter of the instance, or even the diameters of the individual clusters.
Hence, our notion of additive perturbation stability allows the instance to be unbounded.

\paragraph{Geometric properties of $\eps$-APS instances.} Clusters in the optimal solution of an $\eps$-APS instance satisfy a natural geometric condition --- there is an ``angular separation'' between every pair of clusters.

\begin{proposition}[Geometric Implication of $\eps$-APS]
\label{prop:geometric-condition}
Let $X$ be an $\eps$-APS instance and let $C_i, C_j$ be two clusters in its optimal solution. Any point $x\in C_i$ lies in a cone whose axis is along the direction $(\mu_i - \mu_j)$ with half-angle $\text{arctan}(1/\eps)$. 
Hence if $u$ is the unit vector along $\mu_i -\mu_j$ then   
\begin{equation} \label{eq:geometric-condition}
 \forall x \in C_i,~	\frac{\abs{\inner{x-\tfrac{\mu_i+\mu_j}{2},u} }}{\norm{x-\tfrac{\mu_i+\mu_j}{2}}_2}> \frac{\eps}{\sqrt{1+\eps^2}}.
 \end{equation}
\end{proposition}

The distance between $\mu_i$ and the apex of the cone is $\Delta=(\tfrac{1}{2}-\eps)D$. We will call $\Delta$ the \textit{scale parameter} of the clustering. See Figure~\ref{fig:intro}a for an illustration.

\begin{figure}
\centering
\begin{minipage}{0.4\textwidth}
\label{fig:stability_diagrams}
\centering
    \includegraphics[width= 0.9\linewidth]{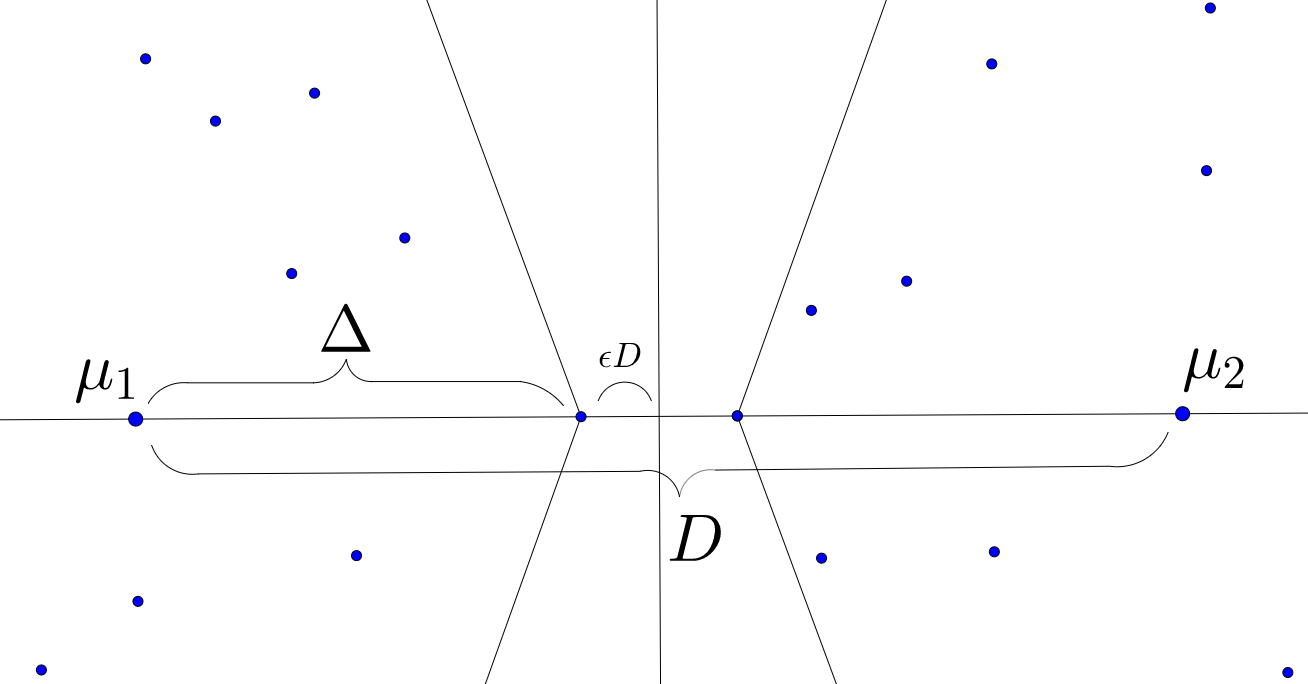}
  \end{minipage}
 \begin{minipage}{0.4\textwidth}
 \centering
    \includegraphics[width= 0.9\linewidth]{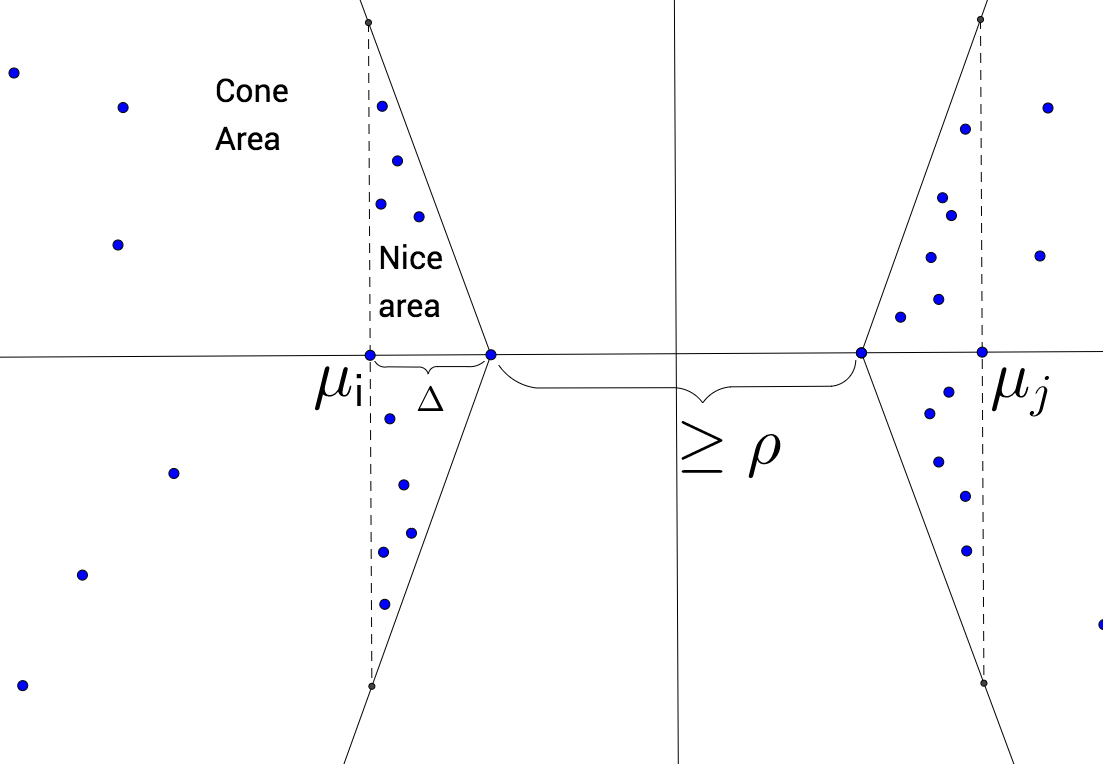}
  \end{minipage}
 \caption{
  \textbf{a.} An $\eps$-APS instance. The means are separated by a distance $D$, the half-angle of each cone is $\text{arctan}(1/\eps)$ and the distance between $\mu_1$ and the apex of the cone $\Delta\leq D/2$.
  \textbf{b.} A $\APSM$ instance with scale parameter $\Delta$. The half-angle of each cone is $\text{arctan}(1/\eps)$ and the distance between the apexes of the cones is at least $\rho$.
}
\label{fig:intro}
\end{figure}

We believe that many clustering instances in practice satisfy the $\eps$-APS condition for reasonable constants $\eps$. In fact, our experiments in Section~\ref{sec:experiments} suggest that the above geometric condition is satisfied for reasonable values e.g., $\eps \in (0.001,0.2)$.
 
While the points can be arbitrarily far away from their own means, the above angular separation \eqref{eq:geometric-condition} is crucial in proving the polynomial time guarantees for our algorithms. For instance, this implies that at least $1/2$ of the points in a cluster $C_i$ are within a Euclidean distance of at most $O(\Delta/\epsilon)$ from $\mu_i$. 
This geometric condition ~\eqref{eq:geometric-condition} of the dataset enables the design of a tractable algorithm for $k=2$ with provable guarantees. This algorithm is based on a modification of the perceptron algorithm in supervised learning, and is inspired by \cite{Blumperceptron}. See Section~\ref{sec:2clustering} for details on the $k=2$ case.
\begin{itheorem}
\label{ithm:twoIsEasy}
For any fixed $\epsilon>0$, there exists a $dn^{\poly(1/\epsilon)}$ time algorithm that correctly clusters all $\epsilon$-APS $2$-means instances.
\end{itheorem}

For $k$-means clustering, similar techniques can be used to learn the separating halfspace for each pair of clusters. However this incurs an exponential dependence on $k^2$, which renders this approach inefficient for large $k$.\footnote{We remark that the results of \cite{BenDavid} also incur an exponential dependence on $k$.} We now consider a natural strengthening of this assumption that allows us to achieve $\poly(n,d,k)$ guarantees for general $k$.

\paragraph{Angular Separation with additional margin separation.} We consider a natural strengthening of additive perturbation stability where there is an additional margin between any pair of clusters. This is reminiscent of margin assumptions in supervised learning of halfspaces and spectral clustering guarantees of Kumar and Kannan~\cite{KK10} (see Section \ref{sub:related_work}). Consider a $k$-means clustering instance $X$ with optimal solution $C_1, C_2, \dots, C_k$. We say this instance is $\APSM$ iff for each $i \ne j \in [k]$, the subinstance induced by $C_i, C_j$ has parameter scale $\Delta$, and all points in the clusters $C_i, C_j$ lie inside cones of half-angle $\text{arctan}(1/\eps)$, which are separated by a margin of at least $\rho$. This is implied by the stronger condition that the subinstance induced by $C_i, C_j$ is $\eps$-additive perturbation stable with scale parameter $\Delta$ even when $C_i$ and $C_j$ are moved towards each other by $\rho$. See Figure~\ref{fig:intro}b for an illustration.
$\APSM$ stable instances are defined formally in geometric terms in Section~\ref{sec:stabilitydef}.

\begin{itheorem}[Polytime algorithm for $\APSM$ instances]
There is an $\widetilde{O}(n^2 kd)$-time\footnote{The $\widetilde{O}$ hides logarithmic factors in $n$. } algorithm that given any instance $X$
that is $\APSM$ with $\rho \ge \Omega(\Delta/\eps^2)$ recovers its optimal clustering $C_1, \dots, C_k$.
\end{itheorem}

A formal statement of the theorem (with unequal sized clusters) and its proof are given in Section~\ref{sec:kmeans}. We prove these polynomial time guarantees for a new, simple algorithm  (Algorithm~\ref{alg1}). The algorithm constructs a graph with one vertex for each point, and edges between points that are within a distance of at most $r$ (for an appropriate threshold $r$). The algorithm then finds the $k$-largest connected components and uses the empirical means of these $k$ components to cluster all the points.

In addition to having provable guarantees, the algorithm also seems efficient in practice, and performs well on standard clustering datasets.
Experiments that we conducted on some standard clustering datasets in UCI suggest that our algorithm manages to almost recover the ground truth and achieves a $k$-means objective cost that is very comparable to Lloyd's algorithm and $k$-means++.

In fact, our algorithm can also be used to initialize Lloyd's algorithm: our guarantees show that when the instance is $\APSM$, one iteration of Lloyd's algorithm already finds the optimal clustering. Experiments suggest that our algorithm finds initializers of smaller $k$-means cost compared to the initializers of $k$-means++ \cite{AV07} and also recover the ground-truth to good accuracy.

Experimental results and analysis of real-world data sets can be found in Section~\ref{sec:experiments}.

\paragraph{Robustness to outliers.}
Perturbation stability requires the optimal solution to remain completely unchanged under any valid perturbation. In practice, the stability of an instance may be dramatically reduced by a few outliers.
We show provable guarantees for a slight modification of Algorithm~\ref{alg1} in the setting where an $\eta$-fraction of the points can be arbitrary outliers, and do not lie in the stable regions.  Formally, we assume that we are given an instance $X \cup Z$ where there is an (unknown) set of points $Z$ with $|Z| = \eta |X|$ such that $X$ is a $\APSM$ instance. Here $\eta n$ is assumed to be less than the size of the smallest cluster by a constant factor. This is similar to robust perturbation resilience considered in \cite{BalcanLiang,MMVstability}. Our experiments in Section~\ref{sec:experiments} indicate that the stability or separation can increase a lot after ignoring a few points close to the margin. 

In what follows, $w_{\max} = \max \abs{C_i}/n$ and $w_{\min}=\min\abs{C_i}/n$ are the maximum and minimum weight of clusters, and $\eta < w_{\min}$.
\begin{itheorem}
\label{ithm:rho_epsilon_beta_eta_is_easy}
Given $X\cup Z$ where $X$ is $\APSM$ for
\[
\rho=\Omega\left(\frac{\Delta}{\epsilon^2}\left(\frac{w_{\max} + \eta}{w_{\min} - \eta}\right)\right)
\]
and $\eta = \abs{Z}/\abs{X}<\wmin$, there is a polynomial time algorithm running in time $\widetilde{O}(n^2 dk)$ that returns a clustering consistent with $C_1,\dots,C_k$ on $X$.
\end{itheorem}

This robust algorithm is effectively the same as Algorithm~\ref{alg1} with one additional step that removes all low-degree vertices in the graph. This step removes bad outliers in $Z$ without removing too many points from $X$.

\subsection{Comparisons to other related work}
\label{sub:related_work}
Awasthi et al. showed that $\gamma$-multiplicative perturbation stable instance also satisfied the notion of $\gamma$-center based stability (every point is a $\gamma$-factor closer to its center than to any other center)~\cite{ABS12}. They showed that an algorithm based on the classic single linkage algorithm works under this weaker notion when $\gamma \ge 3$. This was subsequently improved by \cite{BalcanLiang}, and the best result along these lines~\cite{AMM17} gives a polynomial time algorithm that works for $\gamma\geq2$. A robust version of $(\gamma,\eta)$-perturbation resilience was explored for center-based clustering objectives~\cite{BalcanLiang}.
As such, the notions of additive perturbation stability, and $\APSM$ instances are incomparable to the various notions of multiplicative perturbation stability. Furhter as argued in~\cite{CDNMthesis}, we believe that additive perturbation stability is more realistic for Euclidean clustering problems.

Ackerman and Ben-David\cite{BenDavid} initiated the study of various deterministic assumptions for clustering instances. The measure of stability most related to this work is Center Perturbation (CP) clusterability (an instance is $\delta$-CP-clusterable if perturbing the centers by a distance of $\delta$ does not increase the cost much). A subtle difference is their focus on obtaining solutions with small objective cost\cite{BenDavid}, while our goal is to recover the optimal clustering. However, the main qualitative difference is how the length scale is defined --- this is crucial for additive perturbations. The run time of the algorithm in\cite{BenDavid} is $n^{\poly(k,\diam(X)/\delta)}$, where the length scale of the perturbations is $\diam(X)$, the diameter of the whole instance. Our notion of additive perturbations uses a much smaller length-scale of $\Delta$ (essentially the inter-mean distance; see Prop. 1.1 for a geometric interpretation), and Theorem 1.2 gives a run-time guarantee of $n^{\poly(\Delta/\delta)}$ for $k = 2$ (Theorem 1.2 is stated in terms of $\epsilon=\Delta/\delta$). By using the largest inter-mean distance instead of the diameter as the length scale, our algorithmic guarantees can also handle unbounded clusters with arbitrarily large diameters and outliers.

The exciting results of Kumar and Kannan~\cite{KK10} and Awasthi and Sheffet\cite{AS12} also gave a determinstic margin-separation condition, under which spectral clustering (PCA followed by $k$-means) \footnote{This requires appropriate initializers, that they can obtain in polynomial time.} finds the optimum clusters under deterministic conditions about the data. Suppose $\sigma=\norm{X-C}^2_{op}/n$ is the ``spectral radius'' of the dataset, where $C$ is the matrix given by the centers. In the case of equal-sized clusters, the improved results of \cite{AS12} proves approximate recovery of the optimal clustering if the margin $\rho$ between the clusters along the line joining the centers satisfies $\rho=\Omega(\sqrt{k} \sigma)$.
Our notion of margin $\rho$ in $\APSM$ instances is analogous to the margin separation notion used by the above results on spectral clustering~\cite{KK10,AS12}. In particular, we require a margin of $\rho =\Omega(\Delta/\eps^2)$ where $\Delta$ is our scale parameter, with no extra $\sqrt{k}$ factor.
However, we emphasize that the two margin conditions are incomparable, since the spectral radius $\sigma$ is incomparable to the scale parameter $\Delta$.

We now illustrate the difference between these deterministic conditions by presenting a couple of examples. Consider an instance with $n$ points drawn from a mixture of $k$ Gaussians in $d$ dimensions with identical diagonal covariance matrices with variance $1$ in the first $O(1)$ coordinates and roughly
$1/d$ in the others, and all the means lying in the subspace spanned by these first $O(1)$ co-ordinates. In this setting, the results of~\cite{KK10,AS12} require a margin separation of at least $\sqrt{k\log n}$ between clusters.
On the other hand, these instances satisfy our geometric conditions with $\epsilon=\Omega(1)$, $\Delta~\sqrt{\log n}$ and therefore our algorithm only needs a margin separation of $\rho\sqrt{\log n}$ (hence, saving a factor of $\sqrt k$)\footnote{Further, while algorithms for learning GMM models may work here, adding some outliers far from the decision boundary will cause many of these algorithms to fail, while our algorithm is robust to such outliers.}. However, if the $n$ points were drawn from a mixture of spherical Gaussians in high dimensions (with $d \gg k$), then the margin condition required for~\cite{KK10,AS12} is weaker.

Finally, we note another strand of recent works show that convex relaxations for $k$-means clustering become integral under distributional assumptions about points and sufficient separation between the components~\cite{AfonsoCharikar,Mixonetal}.


\section{Preliminaries}
In the $k$-means clustering problem, we are given $n$ points $X=\set{x_1,\dots,x_n}$ in $\R^d$ and need to find $k$ centers $\mu_1,\dots,\mu_k\in\R^d$ minimizing
\[
	\sum_{x\in X}\min_{i\in[k]}\Norm{x-\mu_i}^2.
\]

A given choice of centers $\mu_1,\dots,\mu_k$ determines an optimal clustering $C_1,\dots,C_k$ where $C_i = \set{x|i=\arg\min_{j}\norm{x-\mu_j}}$. We can rewrite the objective as
\[
  \sum_{i\in[k]}\sum_{x\in C_i} \Norm{x-\mu_i}^2.
\]

On the other hand, a given choice for cluster $C_i$ determines its optimal center as $\mu_i=\frac{1}{\abs{C_i}}\sum_{x\in C_i} x$, the mean of the points in the set. Thus, we can reformulate the problem as minimizing over clusters $C_1,C_2,\dots,C_k$ of $\set{x_i}$ the objective
\[
	\sum_{i\in[k]}\sum_{y\in C_i}\Norm{y-\left(\frac{1}{\abs{C_i}}\sum_{x\in C_i}x\right)}^2.
\]

$k$-means clustering is NP-hard for general Euclidean space $\R^d$ even in the case of $k=2$ \cite{dasgupta2008hardness}.

\section{Stability definitions and geometric properties}
\label{sec:stabilitydef}

\subsection{Balance parameter}
We define an instance parameter, $\beta$, capturing how balanced a given instance's clusters are.

\begin{definition}[Balance parameter]
Given an instance $X$ with optimal clustering $C_1,\dots,C_k$, we say $X$ satisfies balance parameter $\beta\geq 1$ if for all $i\neq j$, $\beta\abs{C_i}>\abs{C_j}$.
\end{definition}

\subsection{Additive perturbation stability}

\begin{definition}[$\epsilon$-additive perturbation]
Let $X=\set{x_1,\dots,x_n}$ be a $k$-means clustering instance with unique optimal clustering $C_1,C_2,\dots,C_k$ whose means are given by $\mu_1,\mu_2,\dots,\mu_k$. Let $D=\max_{i,j}\Norm{\mu_i-\mu_j}$. We say that $X'=\set{x'_1,\dots,x'_n}$ is an $\epsilon$-additive perturbation of $X$ if for all $i$, $\Norm{x'_i-x_i}\leq \epsilon D$.
\end{definition}

\begin{definition}[$\epsilon$-additive perturbation stability]
\label{def:APS}
Let $X$ be a $k$-means clustering instance with unique optimal clustering $C_1,C_2,\dots,C_k$. We say that $X$ is $\epsilon$-additive perturbation stable (APS) if every $\epsilon$-additive perturbation of $X$ has an optimal clustering given by $C_1,C_2,\dots,C_k$.
\end{definition}

Intuitively, the difficulty of the clustering task increases as the stability parameter $\epsilon$ decreases. For example, when $\epsilon=0$ the set of $\epsilon$-APS instances contains any instance with a unique solution. In the following we will only consider $\epsilon>0$.

\subsection{Geometric implication of $\epsilon$-APS}

Let $X$ be an $\epsilon$-APS $k$-means clustering instance such that each cluster has at least $4$ points. Fix $i\neq j$ and consider clusters $C_i$, $C_j$ with means $\mu_i$, $\mu_j$. We fix the following notation.
\begin{itemize}
  \item Let $D_{i,j}=\Norm{\mu_i-\mu_j}$ and let $D=\max_{i',j'}\Norm{\mu_{i'}-\mu_{j'}}$.
  \item Let $u=\frac{\mu_i-\mu_j}{\Norm{\mu_i-\mu_j}}$ be the unit vector in the intermean direction. Let $V=u^\perp$ be the space orthogonal to $u$. For $x\in\R^d$, let $x_{(u)}$ and $x_{(V)}$ be the projections $x$ onto $u$ and $V$.
  \item Let $p=\frac{\mu_i+\mu_j}{2}$ be the midpoint between $\mu_i$ and $\mu_j$.
\end{itemize}

We can establish geometric conditions that $X$ must satisfy by considering different perturbations. As an example, one could move all points in $C_i$ and $C_j$ towards each other in the intermean direction a distance of $\epsilon D$; by assumption no point has crossed the separating hyperplane and thus we can conclude the existence of a margin of width $2\epsilon D$.

A careful choice of a family of perturbations allows us to prove Proposition \ref{prop:geometric-condition}. Consider the perturbation which moves $\mu_i$ and $\mu_j$ in opposite directions orthogonal to $u$ while moving a single point towards the other cluster parallel to $u$ (see figure \ref{fig:perturbation}). The following lemma establishes Proposition~\ref{prop:geometric-condition}.

\begin{figure}
\centering
\includegraphics[width=0.3\linewidth]{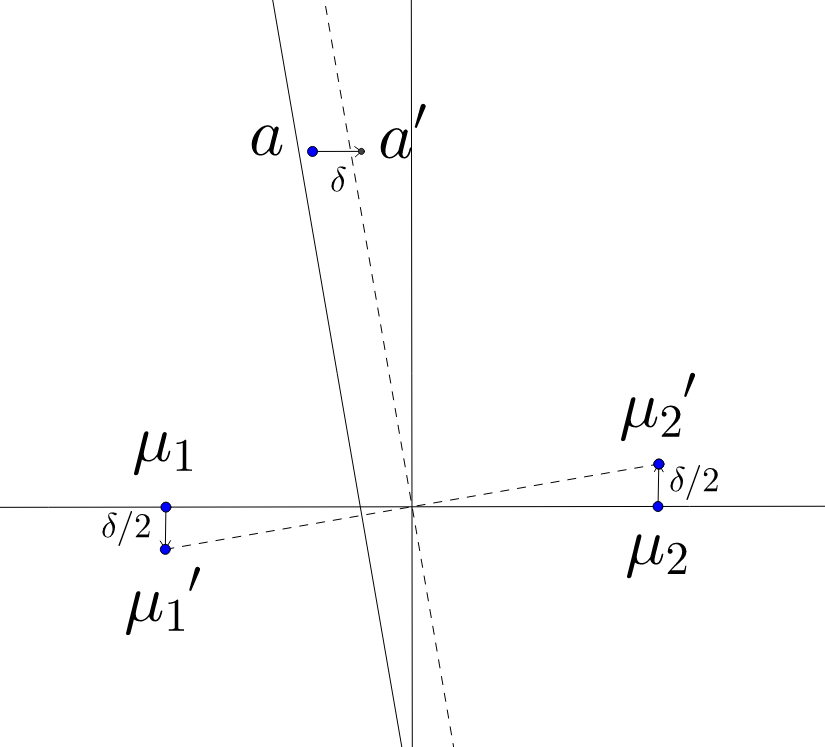}
\caption{An example from the family of perturbations considered by Lemma \ref{lem:marginangularseparation}. Here $v$ is in the upwards direction. If $a$ is to the right of the diagonal solid line, then $a'$ will be to the right of the slanted dashed line and will lie on the wrong side of the separating hyperplane.}
\label{fig:perturbation}
\end{figure}

\begin{lemma}
\label{lem:marginangularseparation}
For any $x\in C_i\cup C_j$, $  \norm{(x-p)_{(V)}}\leq\frac{1}{\epsilon}\left(\norm{(x-p)_{(u)}}-\epsilon D_{i,j}\right)$.
\end{lemma}
\begin{proof}
Let $v\in V$ be a unit vector perpendicular to $u$.
Without loss of generality, let $a\in C_i$ (taking $u$ or $-u$ does not change the inequality). Let $b,c,d\in C_i$ such that $a,b,c,d\in C_i$ are distinct. Let $\delta = \epsilon D_{i,j}\leq \epsilon D$ and consider the $\epsilon$-additive perturbation $X'$ given by the union of
\begin{align*}
  \set{a-\delta u, b + \delta u, c -\delta v, d - \delta v} \cup
  \set{x-\tfrac{\delta}{2}v|x\in C_i\setminus\set{a,b,c,d}}\cup
  \set{x+\tfrac{\delta}{2} v|x\in C_j}
\end{align*}
and an unperturbed copy of $X\setminus (C_i\cup C_j)$.

By assumption, $\set{C_i,C_j}$ remain optimal clusters in $X'$. We have constructed $X'$ such that the new means of $C_i$, $C_j$ are
$\mu_i'=\mu_i-\frac{\delta}{2}v$ and
$\mu_j'=\mu_j+\frac{\delta}{2}v$,
and the midpoint between the means is $p'=p$. The halfspace containing $\mu_i'$ given by the linear separator between $\mu_i'$ and $\mu_j'$ is $\inner{x-p',\mu_i'-\mu_j'}\geq0$. Hence, as $a'$ is classified correctly by the $\epsilon$-APS assumption,
\begin{align*}
\inner{a'-p',\mu_i'-\mu_j'} &= \inner{a-p -\delta u,D_{i,j} u - \delta v}\\
&= D_{i,j}(\inner{a-p, u} - \epsilon\inner{a-p,v} - \delta)\geq 0
\end{align*}
Then noting that $\inner{a-p,u}\geq 0$, we have that $\inner{a-p, v}\leq\frac{1}{\epsilon}\left(\norm{(a-p)_{(u)}}-\delta\right)$.
\end{proof}

This geometric property follows from perturbations which only affect two clusters at a time. Our results follow from this weaker notion.

\subsection{$(\rho,\Delta,\epsilon)$-separation}
Motivated by Lemma \ref{lem:marginangularseparation}, we define a geometric condition where the angular separation and margin separation are parametrized separately. These separations are implied by a stronger stability assumption where any pair of clusters is $\epsilon$-APS with scale parameter $\Delta$ even after being moved towards each other a distance of $\rho$.

We say that a pair of clusters is $(\rho,\Delta,\epsilon)$-separated if their points lie in cones with axes along the intermean direction, half-angle $\arctan(1/\epsilon)$, and apexes at distance $\Delta$ from their means and at least $\rho$ from each other (see figure \ref{fig:intro}b). Formally, we require the following.
\begin{definition}
[Pairwise $(\rho,\Delta,\epsilon)$-separation]
Given a pair of clusters $C_i$, $C_j$ with means $\mu_i$, $\mu_j$, let $u=\frac{\mu_i-\mu_j}{\norm{\mu_i-\mu_j}}$ be the unit vector in the intermean direction and let $p=(\mu_i+\mu_j)/2$. We say that $C_i$ and $C_j$ are $(\rho,\Delta,\epsilon)$-separated if $D_{i,j}\geq \rho+2\Delta$ and for all $x\in C_i\cup C_j$,
  \[
    \norm{(x-p)_{(V)}} \leq \frac{1}{\epsilon}\left(\norm{(x-p)_{(u)}} - (D_{i,j}/2 - \Delta)\right).
  \]
\end{definition}

\begin{definition}
[$(\rho,\Delta,\epsilon)$-separation]
We say that an instance $X$ is $(\rho,\Delta,\epsilon)$-separated if every pair of clusters in the optimal clustering is $(\rho,\Delta,\epsilon)$-separated.
\end{definition}

\section{$k$-means clustering for $k=2$}
\label{sec:2clustering}

In this section, we give an algorithm that is able to cluster $2$-means $\epsilon$-APS instances correctly.

\begin{theorem}
\label{thm:twoIsEasy}
There exists a universal constant $c\geq 1$ such that for any fixed $\epsilon>0$, there exists an $n^{O((1/\epsilon)^c)}d$ time algorithm that correctly clusters all $\epsilon$-APS $2$-means instances.
\end{theorem}

The algorithm is inspired by work in \cite{Blumperceptron} showing that the perceptron algorithm runs in poly-time with high probability in the smoothed analysis setting.

\subsection{Review of perceptron algorithm}
\label{sub:review_of_perceptron_algorithm}

Suppose $y_1,\dots,y_n$ is a sequence of labeled $\set{+1,-1}$-samples consistent with a linear threshold function, i.e., there exists vector $w^*$ such that the labeling function $\ell(y_i)$ is consistent with $\sgn(\inner{y_i,w^*})$. At time $t=0$, the perceptron algorithm sets $w_0=0$. At each subsequent time step, the algorithm sees sample $y_t$, outputs $\sgn(\inner{y_t,w_{t-1}})$ as its guess for $\ell(y_t)$, sees the true label $\ell(y_t)$, and updates $w_t$. On a correct guess, $w_t=w_{t-1}$, and on a mistake $w_t=w_{t-1}+\ell(y_t)y_t/\Norm{y_t}$.

The following well-known theorem \cite{block1962perceptron} bounds the number of total mistakes the perceptron algorithm can make in terms of the sequence's angular margin.

\begin{theorem}
\label{thm:perceptronMistakes}
The number of mistakes made by the perceptron algorithm is bounded above by $(1/\gamma)^2$ for
\[
	\gamma = \min_{i\in[n]}\frac{\abs{\inner{y_i,w^*}}}{\Norm{y_i}\Norm{w^*}}.
\]
\end{theorem}

For a universe $U$ of elements and a function $f:U\to \Z_{\geq 0}$, we will denote by $(U,f)$ the multiset where $u\in U$ appears in the multiset $f(u)$-many times. The size of a multiset is $\sum_{u\in U}f(u)$. The next lemma is an immediate consequence of the above theorem (see proof in Appendix \ref{app:sec:2clustering}).

\begin{lemma}
\label{lem:perceptronDirection}
There exists a multiset $M=(\set{y_1,\dots,y_n},f)$ of size at most $(1/\gamma)^{2}$ such that $\sum_{y\in M} \ell(y)\frac{y}{\Norm{y}}$ correctly classifies all of $\set{y_1,\dots,y_n}$.
\end{lemma}

\subsection{A perceptron-based clustering algorithm}

Fix the following notation: let $X=\set{x_1,\dots,x_n}\subseteq\R^d$ be an $\epsilon$-APS $2$-means clustering instance with optimal clusters $C_1$, $C_2$ such that each cluster has at least $4$ points.
Let $D = \Norm{\mu_1-\mu_2}$, $u=\frac{\mu_1-\mu_2}{\Norm{\mu_1-\mu_2}}$, $p=\frac{\mu_1+\mu_2}{2}$. Without loss of generality, assume that $\sum_i x_i = 0$.

Lemma \ref{lem:marginangularseparation} gives a lower bound for $\gamma$ in the correctly-centered set $\set{x_1-p,\dots,x_n-p}$.
Thus Lemma \ref{lem:perceptronDirection} might suggest a simple algorithm: for each multiset of bounded size and each of its possible labels, compute the cost of the associated clustering, then output the clustering of minimum cost.
However, a difficulty arises as the clusters $C_1$, $C_2$ may not be linearly separable (in particular the separating hyperplane may not pass through the origin). Note that the guarantees of the perceptron algorithm, and hence Lemma \ref{lem:perceptronDirection}, do not hold in this case. Instead, we will apply the above idea to an instance $Y$, constructed from $X$, in which $C_1$, $C_2$ are linearly separable and we can efficiently lower bound $\gamma$.

Consider the following algorithm.
\begin{algorithm}
	\label{alg:perceptronForTwo}
	\hrulefill
	\begin{algorithmic}[1]
		\Require $X=\set{x_1,\dots,x_n}$, $\epsilon$
		\State If necessary, translate $X$ such that $\sum x_i= 0$
		\ForAll {pairs $a,b$ of distinct points in $\set{x_i}$}
			\State Let $\delta=\Norm{a-b}$
			\State Let $Y_{a,b}=\set{y_1,\dots,y_n}$ be an instance given by $y_i = \left(\begin{smallmatrix}
				x_i, &\delta
			\end{smallmatrix}\right)\in\R^{d+1}$
				\ForAll {multisets $M$ of size at most $c_1^{-2}\epsilon^{-8}$ and assignments $\ell:M\to\set{\pm 1}$}
							\State Let $w=\sum_{y\in M} \ell(y)\frac{y}{\Norm{y}}$
								\State Calculate $k$-means cost of $C_1=\set{x_i|\inner{w,y_i}\geq 0},C_2=\set{x_i|\inner{w,y_i}<0}$.
				\EndFor
		\EndFor
			\State Return clustering with smallest $k$-means objective found above
		\end{algorithmic}
		\hrulefill
\end{algorithm}

\subsection{Overview of proof of Theorem \ref{thm:twoIsEasy}}

Each new instance $Y_{a,b}$ constructed in the algorithm has labeling consistent with some linear threshold function: $\ell(y_i)=\ell(x_i)=\sgn(\inner{x_i-p,u})=\sgn(\inner{x_i,u}+\inner{-p,u})$. Then taking $w^*=\left(\begin{smallmatrix}
	u, &\inner{-p,u}/\delta
\end{smallmatrix}\right)$, we have that $\ell(y_i)=\sgn(\inner{y_i,w^*})$.

We will lower bound $\gamma$ for a particular instance $Y_{a,b}$ in which $a,b$ have nice properties.
The following lemma states that on one of the iterations of its outer for loop, Algorithm \ref{alg:perceptronForTwo} will pick such points.
\begin{lemma}
\label{lem:nicearenonempty}
There exist points $a\in C_1$, $b\in C_2$ such that $\inner{a-p,u}\leq \Delta/2$ and $\inner{b-p,-u}\leq \Delta/2$.
\end{lemma}

The geometric conditions implied by $\epsilon$-APS allow us to bound $\delta=\norm{a-b}$ in terms of $\epsilon,D$. In particular, using this handle on $\delta$, it is possible to prove the following lower bound on $\gamma$.

\begin{lemma}
\label{lem:newGamma}
There exists constant $c_1$ such that for any $a,b$ satisfying Lemma \ref{lem:nicearenonempty}, the corresponding instance $Y_{a,b}$ has
\[
\gamma=\min_{i\in[n]}\frac{\abs{\inner{y_i,w^*}}}{\Norm{y_i}\Norm{w^*}}\geq c_1\epsilon^4.
\]
\end{lemma}

The correctness of Algorithm \ref{alg:perceptronForTwo} for all $\epsilon$-APS $2$-means clustering instances in which each cluster has at least $4$ points then follows from Lemmas \ref{lem:perceptronDirection}, \ref{lem:nicearenonempty}, and \ref{lem:newGamma}.
On the other hand, the optimal $2$-means clustering where one of the clusters has at most $3$ points can be calculated in $O(n^4d)$ time.
An algorithm that returns the better of these two solutions thus correctly clusters all $\epsilon$-APS $2$-means instances, completing the proof of Theorem \ref{thm:twoIsEasy}. See Appendix \ref{app:sub:proof_of_2clustering_lemmas} for proofs of Lemmas \ref{lem:nicearenonempty} and \ref{lem:newGamma}.
\section{$k$-means clustering for general $k$}
\label{sec:kmeans}

For general $k$, we will require the stronger $(\rho,\Delta,\epsilon)$-separation. Consider the following algorithm.

\begin{algorithm} \label{alg1}
	\label{alg:kClustering}
	\hrulefill
	\begin{algorithmic}[1]
		\Require $X=\set{x_1,\dots,x_n}$, $k$.
		\ForAll {pairs $a,b$ of distinct points in $\set{x_i}$}
			\State Let $r=\Norm{a-b}$ be our guess for $\rho$
            \Procedure{\texttt{INITIALIZE}}{}
                \State Create graph $G$ on vertices $\set{x_1,\dots,x_n}$ where $x_i$ and $x_j$ have an edge iff $\norm{x_i-x_j}<r$
            	\State Let $a_1,\dots,a_k\in\R^d$ where $a_i$ is the mean of the $i$th largest connected component of $G$   
            \EndProcedure
            \Procedure{\texttt{ASSIGN}}{}
            	\State Let $C_1,\dots,C_k$ be the clusters obtained by assigning each point in $X$ to the closest $a_i$
            \EndProcedure
            \State Calculate the $k$-means objective of $C_1,\dots,C_k$
        \EndFor
        \State Return clustering with smallest $k$-means objective found above
	\end{algorithmic}
	\hrulefill
\end{algorithm}

\begin{theorem}
\label{thm:rho_epsilon_beta_is_easy}
Algorithm \ref{alg:kClustering} recovers $C_1,\dots, C_k$ for any $\APSM$ instance with $\rho = \Omega\left(\frac{\Delta}{\epsilon^2} + \frac{\beta\Delta}{\epsilon}\right)$ and can be implemented in $\widetilde O(n^2kd)$ time.
\end{theorem}

This running time can be achieved by inserting edges into a dynamic graph in order, maintaining connected components and their means using a union-find data structure, and noting that the number of connected components can change at most $n$ times.

In particular, note that this algorithm does not need any prior knowledge of the stability parameters and its running time has no dependence on $\rho$, $\Delta$, or $\epsilon$.

Define the following regions of $\R^d$ for every pair $i,j$.
Given $i,j$, let $C_i,C_j$ be the corresponding clusters with means $\mu_i,\mu_j$. Let $u=\frac{\mu_i-\mu_j}{\norm{\mu_i-\mu_j}}$ be the unit vector in the inter-mean direction.

\begin{definition}\leavevmode
\begin{itemize}
    \item $\cone{i,j} = \set{x\in\R^d| \norm{(x-(\mu_i-\Delta u))_{(V)}} \leq \frac{1}{\epsilon} \inner{x-(\mu_i-\Delta u), u}  }$,
    \item $\nice{i,j}=\set{x \in \cone{i,j} | \inner{x-\mu_i,u}\leq 0 }$,
    \item $\good{i} = \bigcap_{j\neq i} \nice{i,j}$.
\end{itemize}
\end{definition}

See Figure~\ref{fig:intro}b.\ for an illustration.

It suffices to prove the following two lemmas. Lemma \ref{lem:initialize_is_correct} states that the initialization returned by the \texttt{INITIALIZE} subroutine satisfies certain properties when we guess $r=\rho$ correctly.
As $\rho$ is only used as a threshold on edge lengths, testing the distances between all pairs of data points i.e. $\set{\norm{a-b}: a,b \in X}$ suffices.
Lemma \ref{lem:assign_is_correct} states that the \texttt{ASSIGN} subroutine correctly clusters all points given an initialization satisfying these properties.

\begin{lemma}
\label{lem:initialize_is_correct}
For a $\APSM$ instance with balance parameter $\beta$ and $\rho = \Omega(\beta\Delta/\epsilon)$, the \texttt{INITIALIZE} subroutine finds a set $\set{a_1,\dots,a_k}$ where $a_i\in\good{i}$ when $r=\rho$.
\end{lemma}

\begin{lemma}
\label{lem:assign_is_correct}
For a $\APSM$ instance with $\rho=\Omega(\Delta/\epsilon^2)$, the \texttt{ASSIGN} subroutine recovers $C_1, C_2, \cdots C_k$ correctly when initialized with $k$ points $\set{a_1,a_2,\dots,a_k}$ where $a_i\in \good{i}$.
\end{lemma}

\subsection {Proof of Lemma \ref{lem:initialize_is_correct}.}
Suppose $r=\rho$ and consider the graph constructed by Algorithm \ref{alg:kClustering}.
We start by defining the \textit{core region} of each cluster. 
\begin{definition}[$\core{}$]
Let $\core{i} = \set{x\in\R^d | \norm{x-\mu_i}\leq\Delta/\epsilon}$.
\end{definition}
The core regions are defined in such a way that for each cluster $C_i$, all points in $C_i\cap\core{i}$ belong to a single connected component.
Although $\core{i}$ may not contain too many points on its own, the connected component containing $\core{i}$ will contain most (at least $\beta/(1+\beta)$ fraction) of the points in $C_i$.
Hence, the $k$ largest components will be the connected components containing the $k$ different core regions. Finally, since the connected component containing $\core{i}$ contains most of the points in $C_i$, the geometric conditions of $(\rho,\Delta,\epsilon)$-separation ensure that the empirical mean of the connected component lies in $\good{i}$.
The following lemma states some properties of the connected components in our graph. Its proof can be found in Appendix~\ref{app:sub:proof_of_conn_comp}.

\begin{lemma}
\label{lem:conn_comp_properties}
\leavevmode
\begin{enumerate}
\item Any connected component only contains points from a single cluster.
\item For all $i,j$, $\core{i}\supseteq \nice{i,j}$. There is a point $x\in C_i$ such that $x\in\core{i}\cap \nice{i,j}$.
\item For all $i,j$, let $A_{i,j}=\set{x \in C_i | \inner{x-\mu_i,u} \leq \beta\Delta }$. Then, $\abs{A_{i,j}}\geq \frac{\beta}{1+\beta}\abs{C_i}$.
\item For all $i$, $\core{i}\cap X$ is connected in $G$.
\item For all $i,j$, $A_{i,j}$ is connected in $G$.
\item The largest component, $K_i$, in each cluster contains $A_{i,j}$ for each $j\neq i$. In particular, $\abs{K_i}\geq \frac{\beta}{1+\beta}\abs{C_i}$, and $K_i$ contains $\core{i}\cap X$.
\end{enumerate}
\end{lemma}

Lemma \ref{lem:largest_in_diff} states that the $k$ largest components (and hence $\set{a_1,\dots,a_k}$) must belong to different clusters while Lemma \ref{lem:mean_in_good} states that each $a_i$ lie inside a good region. Together, they imply Lemma \ref{lem:initialize_is_correct}, i.e. each $a_i$ comes from a different good region.

\begin{lemma}
\label{lem:largest_in_diff}
The set of $k$ largest components of $G$ contains the largest component of each cluster.
\end{lemma}
\begin{proof}
Let $K_i$ be the largest component in $C_i$ and let $K_j'$ be a component in $C_j$ that is not the largest. Then by the $\beta$ parameter, $\abs{K_i}\geq \frac{\beta}{1+\beta}\abs{C_i}> \frac{1}{1+\beta}\abs{C_j}\geq\abs{K_j'}$. It follows that the $k$ largest connected components are $K_1,K_2,\dots,K_k$.
\end{proof}

\begin{lemma}
\label{lem:mean_in_good}
The mean of points in $K_i$ lies in $\good{i}$.
\end{lemma}
\begin{proof}
Let $a_i$ be the mean of the points in $K_i$.
As $K_i\subseteq \cone{i,j}$ is a convex set, $a_i\in\cone{i,j}$. As $K_i\supseteq \core{i}\cap X\supseteq \nice{i,j}\cap X$, the points $x\in C_i$ not contained in $K_i$ have $\inner{x-\mu_i, u}>0$. 
Noting that $\sum_{x\in C_i} \inner{x-\mu_i,u}=0$, it follows that $\inner{a_i-\mu_i}\leq 0$. Hence, $a_i\in\nice{i,j}$.
As this holds for each $j\neq i$, $a_i\in\good{i}$.
\end{proof}

\subsection {Proof of Lemma \ref{lem:assign_is_correct}.}
We will show that for any $a_i\in \nice{i,j}$, $a_j\in \nice{j,i}$, and $x\in C_i$, $x$ is closer to $a_i$ than to $a_j$. 
The following lemma states some properties of the perpendicular bisector between $a_i$ and $a_j$. These statements follow from the definitions of the nice regions and the angular separation. Its proof can be found in Appendix~\ref{app:sub:proof_of_bisector_properties}.

\begin{lemma}
\label{lem:bisectorproperties}
Suppose $\rho =\Omega(\Delta/\epsilon^2)$.
Then, for $a_i\in\nice{i,j}$ and $a_j\in\nice{j,i}$, we have
\begin{enumerate}
    \item $\norm{(a_i-a_j)_{(u)}} \geq \frac{\norm{(a_i-a_j)_{(V)}}}{\epsilon}$,
    \item $\inner{\frac{a_i+a_j}{2} - p, u}\leq \frac{\Delta}{2}$, and
    \item $\Bignorm{\left(\frac{a_i+a_j}{2} - p\right)_{(V)}} \leq \Delta/\epsilon$.
\end{enumerate}
\end{lemma}

To prove Lemma \ref{lem:assign_is_correct}, we rewrite the condition $\norm{x-a_i} \leq \norm{x-a_j}$ as $\inner{x-p - (\tfrac{1}{2}(a_i+a_j) - p),a_i-a_j} \geq 0$. Then we write each vector in terms of their projection on $u$ and $V$ and use the above lemma to bound each of the terms.

\begin{proof}[Proof of Lemma~\ref{lem:assign_is_correct}]
It suffices to show that for any $a_i\in \nice{i,j}$, $a_j\in\nice{j,i}$, and $x\in C_i$, $\norm{x-a_i} \leq \norm{x-a_j}$.
Then by Lemma \ref{lem:bisectorproperties} above,
\begin{align*}
\Iprod{(x-p) - \left(\frac{a_i+a_j}{2}-p\right), a_i-a_j} &= \Iprod{(x-p)_{(u)}, (a_i-a_j)_{(u)}} + \Iprod{(x-p)_{(V)}, (a_i-a_j)_{(V)}}\\
&\hphantom{=}\hspace{2em}- \Iprod{(\tfrac{1}{2}(a_i + a_j)-p)_{(u)}, (a_i-a_j)_{(u)}}\\
&\hphantom{=}\hspace{2em} - \Iprod{(\tfrac12(a_i + a_j)-p)_{(V)}, (a_i-a_j)_{(V)}}\\
&\geq\norm{(x-p)_{(u)}}\norm{(a_i-a_j)_{(u)}} - \frac{1}{\epsilon}\left(\norm{(x-p)_{(u)}} - \rho/2\right)\epsilon\norm{(a_i-a_j)_{(u)}}\\
&\phantom{\geq}\hspace{2 em} - \frac{\Delta}{2}\norm{(a_i-a_j)_{(u)}} - \frac{\Delta}{\epsilon}\epsilon \norm{(a_i-a_j)_{(u)}}\\
&=\left(\frac{\rho}{2} - \frac{3}{2}\Delta\right)\norm{(a_i-a_j)_{(u)}}\geq 0
\end{align*}
where the first inequality follows because of equality on the first term and Cauchy-Schwarz on the rest. So, for all $a_i\in \nice{i,j}$, $a_j\in\nice{j,i}$, and $x\in C_i$, $x$ is closer to $a_i$ than $a_j$.
\end{proof}


\section{Robust $k$-means}
A simple extension of algorithm \ref{alg:kClustering} does well even in the presence of adversarial noise for instances with $(\rho,\Delta,\epsilon)$-separation for large enough $\rho$. Specifically, we consider the following model.

Let $X=\set{x_1,\dots,x_n}\subset \R^d$ be a $k$-means clustering instance with optimal clustering $C_1,\dots,C_k$. We call $X$ the set of \textit{pure} points. An additional set of at most $\eta n$-many \textit{impure} points $Z\subset\R^d$ is added by an adversary. Our goal is to find a clustering of $X\cup Z$ that agrees with $C_1,\dots,C_k$ on the pure points.

Let $w_{\max} = \max \abs{C_i}/n$ and let $w_{\min}=\min\abs{C_i}/n$ be the maximum and minimum weight of clusters. 
We will assume that $\eta < w_{\min}$.

\begin{algorithm}
	\label{alg:robustKClustering}
	\hrulefill
	\begin{algorithmic}[1]
		\Require $X\cup Z$, $r$, $t$
        \Procedure{\texttt{INITIALIZE}}{}
            \State Create graph $G$ on $X\cup Z$ where vertices $u$ and $v$ have an edge iff $\norm{u-v}<r$
            \State Remove vertices with vertex degree $< t$
            \State Let $a_1,\dots,a_k\in\R^d$ where  $a_i$ is the mean of the $i$th largest connected component of $G$
        \EndProcedure
        \Procedure{\texttt{ASSIGN}}{}
            \State Let $C_1,\dots,C_k$ be the clusters obtained by assigning each point in $I\cup Z$ to the closest $a_i$
        \EndProcedure
	\end{algorithmic}
	\hrulefill
\end{algorithm}

\begin{theorem}
\label{thm:rho_epsilon_beta_eta_is_easy}
Given $X\cup Z$ where $X$ satisfies $(\rho,\Delta,\epsilon)$-separation for
\[
\rho=\Omega\left(\frac{\Delta}{\epsilon^2}\left(\frac{w_{\max} + \eta}{w_{\min} - \eta}\right)\right),
\]
$\abs{X} = n$ and $\abs{Z}\leq \eta n$ for $\eta<w_{\min}$, there exists values of $r,t$ such that Algorithm \ref{alg:robustKClustering} returns a clustering consistent with $C_1,\dots,C_k$ on $X$. Algorithm \ref{alg:robustKClustering} can be implemented in $\widetilde O(n^2kd)$ time.
\end{theorem}

The proof of this theorem is similar to the proof of Theorem \ref{thm:rho_epsilon_beta_is_easy} and can be found in Appendix~\ref{app:sec:robust}.
\section{Experimental results}
\label{sec:experiments}

We evaluate Algorithm \ref{alg:kClustering} on multiple real world datasets and compare its performance to the performance of $k$-means++, and also check how well these datasets satisfy our geometric conditions.

\paragraph{Datasets.}
Experiments were run on unnormalized and normalized versions of four labeled datasets from the UCI Machine Learning Repository: Wine ($n=178 $, $k=3$, $d=13$), Iris ($n=150 $, $k=3$, $d=4$), Banknote Authentication ($n=1372$, $k=2$, $d=5$), and Letter Recognition ($n=20,000$, $k=26$, $d=16$). Normalization was used to scale each feature to unit range.

\paragraph{Performance.}
The cost of the solution returned by Algorithm \ref{alg:kClustering} for each of the normalized and unnormalized versions of the  datasets is recorded in Table \ref{tab:costs} column 2.
Our guarantees show that under $(\rho,\Delta, \epsilon)$-separation for appropriate values of $\rho$ (see section \ref{sec:kmeans}), the algorithm will find the optimal clustering after a single iteration of Lloyd's algorithm. Even when $\rho$ does not satisfy our requirement, we can use our algorithm as an initialization heuristic for Lloyd's algorithm.
We compare our initialization with the $k$-means++ initialization heuristic ($D^2$ weighting). In Table \ref{tab:costs}, this is compared to the smallest initialization cost of 1000 trials of $k$-means++ on each of the datasets, the solution found by Lloyd's algorithm using our initialization and the smallest $k$-means cost of 100 trials of Lloyd's algorithm using a $k$-mean++ initialization.

\paragraph{Separation in real data sets.}
As the ground truth clusterings in our datasets are not in general linearly separable, we consider the clusters given by Lloyd's algorithm initialized with the ground truth solutions.

{\em Values of $\epsilon$ for Lemma \ref{lem:marginangularseparation}.}
We calculate the maximum value of $\epsilon$ such that every pair of clusters satisfies the angular and margin separations implied by $\epsilon$-APS (Lemma \ref{lem:marginangularseparation}).
The results are recorded in Table \ref{tab:epsilon}. We see that the average value of $\epsilon$ lies approximately in the range $(0.01,0.1)$.

{\em Values of $(\rho,\Delta,\epsilon)$-separation.}
We attempt to measure the values of $\rho$, $\Delta$, and $\epsilon$ in the datasets. For $\eta=0.05,0.1$, $\epsilon=0.1,0.01$, and a pair of clusters $C_i$, $C_j$, we calculate $\rho$ as the maximum margin separation a pair of axis-aligned cones with half-angle $\arctan(1/\epsilon)$ can have while capturing a $(1-\eta)$-fraction of all points. For some datasets and values for $\eta$ and $\epsilon$, there may not be any such value of $\rho$, in this case we leave the corresponding entry blank. These results are collected in Table \ref{tab:rho_delta_epsilon_eta_full}.

\paragraph{Ground truth recovery.} The clustering returned by our algorithm recovers well ($\approx 97\%$) the solution returned by Lloyd's algorithm initialized with the ground truth for Wine, Iris, and Banknote Authentication across normalized and unnormalized datasets.

\begin{table}
\caption{Comparison of $k$-means cost for Alg \ref{alg:kClustering} and $k$-means++}
\label{tab:costs}
\begin{center}
\begin{tabular}{lcccc}
\hline
Dataset & Alg \ref{alg:kClustering} & $k$-means++ & Alg \ref{alg:kClustering} with Lloyd's & $k$-means++ with Lloyd's\\
\hline
Wine & 2.376e+06 & 2.426e+06 & 2.371e+06 & 2.371e+06\\
Wine (normalized) & 48.99 & 65.50 & 48.99 & 48.95\\
Iris & 81.04 & 86.45 & 78.95 & 78.94\\
Iris (normalized)& 7.035 & 7.676 & 6.998 & 6.998 \\
Banknote Auth.  & 44808.9 & 49959.9 & 44049.4 & 44049.4\\
Banknote (norm.) & 138.4 & 155.7 & 138.1 & 138.1\\
Letter Recognition &744707& 921643 & 629407 & 611268\\
Letter Rec. (norm.) & 3367.8 & 4092.1 &2767.5 & 2742.3\\
\hline
\end{tabular}
\end{center}
\end{table}

\begin{table}
\caption{Values of $\epsilon$ satisfying Lemma \ref{lem:marginangularseparation}}
\label{tab:epsilon}
\begin{center}
 \begin{tabular}{lccc} 
 \hline
Dataset & Minimum $\epsilon$ & Average $\epsilon$ & Maximum $\epsilon$ \\\hline
Wine & 0.0115 & 0.0731 & 0.191\\
Wine (normalized) & 0.000119 & 0.0394 & 0.107\\
Iris & 0.00638 & 0.103 & 0.256\\
Iris (normalized) & 0.00563 & 0.126 & 0.237\\
Banknote Auth. & 0.00127 & 0.00127 & 0.00127\\
Banknote (norm.) & 0.00175 & 0.00175 & 0.00175\\
Letter Recognition & 3.22e-05 & 0.0593 & 0.239\\
Letter Rec. (norm.) & 8.49e-06 & 0.0564 & 0.247\\\hline
\end{tabular}
\end{center}
\end{table}

\begin{table}
\caption{Values of $(\rho,\epsilon,\Delta)$ satisfied by $(1-\eta)$-fraction of points}
\label{tab:rho_delta_epsilon_eta_full}
\begin{center}
 \begin{tabular}{lccccc} 
 \hline
 Dataset & $\eta$ & $\epsilon$ & minimum $\rho/\Delta$ & average $\rho/\Delta$ & maximum $\rho/\Delta$ \\
 \hline
\multirow{4}{*}{Wine} & \multirow{2}{2em}{0.05} & 0.1 & 0.355 & 0.992 & 2.19\\
& & 0.01 & 0.374 & 1 & 2.2\\\cline{2-6}
& \multirow{2}{*}{0.1} & 0.1 & 0.566 & 1.5 & 3.05\\
& & 0.01 & 0.609 & 1.53 & 3.07\\\hline
\multirow{4}{*}{Wine (normalized)} & \multirow{2}{2em}{0.05} & 0.1 &  &   & \\
& & 0.01 & 0.399 & 1.06 & 2.29\\\cline{2-6}
& \multirow{2}{*}{0.1} & 0.1 & 0.451 & 1.3 & 2.66\\
& & 0.01 & 0.735 & 1.96 & 3.62\\\hline
\multirow{4}{*}{Iris} & \multirow{2}{2em}{0.05} & 0.1 & 0.156 & 2.47 & 5.37\\
& & 0.01 & 0.263 & 2.88 & 6.43\\\cline{2-6}
& \multirow{2}{*}{0.1} & 0.1 & 0.398 & 4.35 & 7.7\\
& & 0.01 & 0.496 & 5.04 & 9.06\\\hline
\multirow{4}{*}{Iris (normalized)} & \multirow{2}{2em}{0.05} & 0.1 & 0.0918 & 1.89 & 3.08\\
& & 0.01 & 0.213 & 2.21 & 3.4\\\cline{2-6}
& \multirow{2}{*}{0.1} & 0.1 & 0.223 & 3.74 & 7.12\\
& & 0.01 & 0.391 & 4.42 & 8.3\\\hline
\multirow{4}{*}{Banknote Auth.} & \multirow{2}{2em}{0.05} & 0.1 & 0.0731 & 0.0731 & 0.0731\\
&&  0.01 & 0.198 & 0.198 & 0.198\\\cline{2-6}
&\multirow{2}{2em}{0.1} & 0.1 & 0.264 & 0.264 & 0.264\\
& & 0.01 & 0.398 & 0.398 & 0.398\\\hline
\multirow{4}{*}{Banknote (norm.)} & \multirow{2}{2em}{0.05} & 0.1 & &   & \\
& & 0.01 & 0.197 & 0.197 & 0.197\\\cline{2-6}
&\multirow{2}{2em}{0.1} & 0.1 & 0.246 & 0.246 & 0.246\\
& & 0.01 & 0.474 & 0.474 & 0.474\\\hline
\multirow{4}{*}{Letter Recognition} & \multirow{2}{2em}{0.05} & 0.1 &&  &\\
& & 0.01 & 0.168 & 2.06 & 6.96\\\cline{2-6}
&\multirow{2}{2em}{0.1} & 0.1 & 0.018 & 2.19 & 7.11\\
& & 0.01 & 0.378 & 3.07 & 11.4\\\hline
\multirow{4}{*}{Letter Rec. (norm.)} & \multirow{2}{2em}{0.05} & 0.1 & & & \\
& & 0.01 & 0.157 & 1.97 & 7.14\\\cline{2-6}
&\multirow{2}{2em}{0.1} & 0.1 & &  &\\
& & 0.01 & 0.378 & 2.92 & 11.2\\\hline
\end{tabular}
\end{center}
\end{table}

\section{Acknowledgments}
The authors are grateful to Avrim Blum for numerous helpful discussions regarding the perceptron algorithm.

\clearpage
\newpage 
{\small
\bibliographystyle{plain}
\bibliography{clustering_stable_instances_bib}
}

\newpage
\begin{appendices}

\section{$k$-means clustering for $k=2$}
\label{app:sec:2clustering}

\subsection{Proof of Lemma \ref{lem:perceptronDirection}}
\label{app:sub:proof_of_perceptron_direction}
\begin{lemma*}
There exists a multiset $M=(\set{y_1,\dots,y_n},f)$ of size at most $(1/\gamma)^{2}$ such that $\sum_{y\in M} \ell(y)\frac{y}{\Norm{y}}$ correctly classifies all of $\set{y_1,\dots,y_n}$.
\end{lemma*}
\begin{proof}
Let $r=(1/\gamma)^2 + 1$.
Consider the performance of the perceptron algorithm on $r$ consecutive runs of the $y_1,\dots,y_n$, i.e., let the input be \[
  	\underbrace{\overbrace{y_1,\dots,y_n}^{1\text{ run}},y_1,\dots,y_n,\dots,y_1,\dots,y_n}_{r\text{ runs}}.
\]
A mistake can only be made on a given run if mistakes were made on every previous run. Suppose the perceptron algorithm makes a mistake on the $r$th run, then the algorithm must have made at least $(1/\gamma)^2+1$ mistakes, a contradiction. Hence the direction of $w$ after $r-1$ runs correctly classifies all of $\set{y_1,\dots,y_n}$.
The value of $w$ is $\sum_{i\in[n]}f(y_i)\ell(y_i)\frac{y_i}{\Norm{y_i}}$ where $f(y_i)$ is the number of times $y_i$ was misclassified.
\end{proof}

\subsection{Proof of Lemmas \ref{lem:nicearenonempty}, \ref{lem:newGamma}}
\label{app:sub:proof_of_2clustering_lemmas}

We state two lemmas that follow immediately from Lemma \ref{lem:marginangularseparation} and will be useful for the proofs in this section.

\begin{lemma}
\label{lem:marginseparation}
For any $x\in X$,
\[
  \norm{\inner{x-p,u}}\geq \epsilon D.
\]
In particular, for $x\in C_1$, $\inner{x-p,u}\geq\epsilon D$ and for $x\in C_2$, $\inner{x-p,u}\leq-\epsilon D$.
\end{lemma}

\begin{lemma}
\label{lem:angularseparation}
For any $x\in X$,
\[
  \frac{\abs{\inner{x-p, u}}}{\Norm{x-p}}\geq \sqrt{\frac{\epsilon^2}{1+\epsilon^2}}.
\]
\end{lemma}

\subsubsection*{Lemma \ref{lem:nicearenonempty}}
We restate and prove Lemma \ref{lem:nicearenonempty} below.
\begin{lemma*}
There exist points $a\in C_1$, $b\in C_2$ such that $\inner{a-p,u}\leq \Delta/2$ and $\inner{b-p,-u}\leq \Delta/2$.
\end{lemma*}
\begin{proof}[Proof of Lemma \ref{lem:nicearenonempty}]
Note that $\inner{\mu_1-p,u}=\frac{1}{\abs{C_1}}\sum_{x\in C_1}\inner{x-p,u}$. As $\inner{\mu_1-p,u}=\Delta/2$, there must be some $a\in C_1$ such that $\set{a-p,u}\leq \Delta/2$. The second assertion is proved similarly.
\end{proof}

\subsubsection*{Lemma \ref{lem:newGamma}}
Note that Lemmas \ref{lem:marginseparation} and \ref{lem:nicearenonempty} together imply that we \textit{cannot} have an instance with $\epsilon > 1/2$.
\begin{lemma}
  \label{lem:noLargeEpsilonInstances}
  There is no $\epsilon$-APS $k$-means clustering instance for $\epsilon> 1/2$.
\end{lemma}

The following lemma bounds $\delta=\norm{a-b}$ in terms of $\epsilon$, $D$.

\begin{lemma}
\label{lem:deltaApprox}
Let $a,b\in X$ be points satisfying Lemma \ref{lem:nicearenonempty}. Then,
\[
(2\epsilon) D\leq\Norm{a-b}\leq\left(\sqrt{\frac{1+\epsilon^2}{\epsilon^2}}\right) D.
\]
\end{lemma}
\begin{proof}
For the first inequality, $\Norm{a-b}\geq \abs{\inner{u,a-b}}=\abs{\inner{u,a-p}-\inner{u,b-p}}$. Then by Lemma \ref{lem:marginseparation}, $\norm{a-b}\geq 2\epsilon D$.

For the second inequality, $\Norm{a-b}\leq\Norm{a-p}+\Norm{p-b}$.
By assumption, $\inner{a-p,u}\leq \Delta/2$. Then by Lemma \ref{lem:angularseparation},
$\Norm{a-p}\leq\sqrt{(1+\epsilon^2)/\epsilon^2} D/2$.
Similarly, $\Norm{b-p}\leq\sqrt{(1+\epsilon^2)/\epsilon^2}D/2$.
\end{proof}

Finally, we restate and prove Lemma \ref{lem:newGamma} below.

\begin{lemma*}
There exists constant $c_1$ such that for any $a,b$ satisfying Lemma \ref{lem:angularseparation}, the corresponding instance $Y_{a,b}$ has
\[
\gamma=\min_{i\in[n]}\frac{\abs{\inner{y_i,w^*}}}{\Norm{y_i}\Norm{w^*}}\geq c_1\epsilon^4.
\]
\end{lemma*}
\begin{proof}
We bound each term in the minimization individually. Let $i\in[n]$, then
\begin{align*}
\frac{\abs{\inner{y_i,w^*}}}{\Norm{y_i}\Norm{w^*}}&= \frac{\abs{\inner{x_i-p,u}}}{\sqrt{\Norm{x_i}^2+\delta^2}\sqrt{1+\left(\frac{\inner{p,u}}{\delta}\right)^2}}.
\end{align*}
We first observe the following facts.
\begin{itemize}
	\item From Lemma \ref{lem:angularseparation}, $\abs{\inner{x_i-p,u}}\geq\sqrt{\frac{\epsilon^2}{1+\epsilon^2}}\Norm{x_i-p}\geq\frac{\epsilon}{1+\epsilon} \Norm{x_i-p}$
	\item By Lemma \ref{lem:angularseparation}, $\Norm{x_i}^2\leq 2\Norm{x_i-p}^2 + 2\Norm{p}^2\leq 2\Norm{x_i-p}^2 +\frac{1}{2}\frac{1+\epsilon^2}{\epsilon^2}{D^2}$
	\item From Lemma \ref{lem:deltaApprox}, $\delta^2 \leq \frac{1+\epsilon^2}{\epsilon^2} D^2$
	\item As $p$ and the origin both lie on the line between $\mu_1$ and $\mu_2$,	$\abs{\inner{p,u}}\leq \frac D 2\leq\frac{\delta}{4\epsilon}$
	\item From Lemma \ref{lem:marginseparation}, $\Norm{x_i-p}\geq \epsilon D$
\end{itemize}
Making each of the substitutions above,
\begin{align*}
\frac{\abs{\inner{y_i,w^*}}}{\Norm{y_i}\Norm{w^*}}&\geq
\epsilon \frac{\Norm{x_i-p}}{(1+\epsilon)\sqrt{2\Norm{x_i-p}^2+\frac{3}{2}\frac{1+\epsilon^2}{\epsilon^2} D^2}\sqrt{1+\frac{1}{16\epsilon^2}}}\\
&\geq \epsilon
\frac{1}{(1+\epsilon)
\sqrt{2+\frac{3}{2}\frac{1+\epsilon^2}{\epsilon^2}\left(\frac{D}{\Norm{x_i-p}}\right)^2}
\sqrt{1+\frac{1}{16\epsilon^2}}
}\\
&\geq \epsilon
\frac{1}{
(1+\epsilon)\sqrt{2+\frac{3}{2\epsilon^2}+\frac{3}{2\epsilon^4}}
\sqrt{1+\frac{1}{16\epsilon^2}}
}.
\end{align*}
Then, completing both squares,
\begin{align*}
\frac{\abs{\inner{y_i,w^*}}}{\Norm{y_i}\Norm{w^*}} &\geq \epsilon
\frac{1}{
(1+\epsilon)
\left(\sqrt{2} + \frac{\sqrt{3/2}}{\epsilon^2}\right)
\left(1+\frac{1/4}{\epsilon}\right)
}\\
&= \epsilon^4\frac{1}{(1+\epsilon)\left(\sqrt{2}\epsilon^2 + \sqrt{3/2}\right)\left(\epsilon+1/4\right)}
\end{align*}

As $\epsilon\leq 1/2$ by Lemma \ref{lem:noLargeEpsilonInstances}, we can bound the fraction below by some constant $c_1\approx 0.563$.
\end{proof}

\section{$k$-means clustering for general $k$}
\label{app:kclustering}

\subsection{Proof of Lemma \ref{lem:conn_comp_properties}}
\label{app:sub:proof_of_conn_comp}

We restate and prove Lemma~\ref{lem:conn_comp_properties} below.
\begin{lemma*}
\leavevmode
\begin{enumerate}
\item Any connected component only contains points from a single cluster.
\item For all $i,j$, $\core{i}\supseteq \nice{i,j}$. There is a point $x\in C_i$ such that $x\in\core{i}\cap \nice{i,j}$.
\item For all $i,j$, let $A_{i,j}=\set{x \in C_i | \inner{x-\mu_i,u} \leq \beta\Delta }$. Then, $\abs{A_{i,j}}\geq \frac{\beta}{1+\beta}\abs{C_i}$.
\item For all $i$, $\core{i}\cap X$ is connected in $G$.
\item For all $i,j$, $A_{i,j}$ is connected in $G$.
\item The largest component, $K_i$, in each cluster contains $A_{i,j}$ for each $j\neq i$. In particular, $\abs{K_i}\geq \frac{\beta}{1+\beta}\abs{C_i}$, and $K_i$ contains $\core{i}\cap X$.
\end{enumerate}
\end{lemma*}

\begin{proof}
\leavevmode
\begin{enumerate}
\item Let $x\in C_i$ and $y\in C_j$. Then $\norm{x-y}\geq \abs{\inner{x-y, u}}\geq \rho$, thus no edge connecting points in different clusters is added to $G$.

\item For $x\in \nice{i,j}$, $\norm{(x-\mu_i)_{(V)}}\leq\frac{1}{\epsilon}(\Delta - \norm{(x-\mu_i)_{(u)}})$, hence $\norm{x-\mu_i} \leq\Delta/\epsilon$. 
Recall $\mu_i$ is the mean of the points in cluster $C_i$.
By an averaging argument, $\nice{i,j}\cap X = \set{x\in C_i| \inner{x-(\mu_i-\Delta u), u}\leq \Delta}$ is nonempty and hence $\core{i}\cap \nice{i,j}$ is nonempty.

\item $\mu_i$ is the mean of the points in cluster $C_i$. By an averaging argument, $\abs{A_{i,j}}\Delta - (\abs{C_i}-\abs{A_{i,j}})\beta\Delta \geq 0$. Rearranging, $\abs{A_{i,j}}\geq \frac{\beta}{1+\beta}\abs{C_i}$.

\item For $x,y\in\core{i}$, $\norm{x-y}\leq 2\Delta/\epsilon$. Thus for $\rho = \Omega(\Delta/\epsilon)$, the points $\core{i}\cap X$ are connected.

\item From 2 above, $\nice{i,j}\cap X$ is nonempty; fix such a point $x$. For $y\in A_{i,j}$, $\norm{x-y}^2 = \norm{(x-y)_{(u)}}^2 + \norm{(x-y)_{(V)}}^2 \leq \left((\beta+1)\Delta\right)^2 + \left((\beta+1)\Delta/\epsilon\right)^2$. Thus for $\rho = \Omega(\beta\Delta/\epsilon)$, all of $A_{i,j}$ is connected through $x$.

\item Let $K_i$ be the component containing $\core{i}\cap X$. By 2 above, for all $j$ there exists a point $x_{(j)}\in\core{i}$ such that $x_{(j)}\in\nice{i,j}\subseteq A_{i,j}$. Then as $A_{i,j}$ is connected, $K_i$ must also contain $A_{i,j}$. As $\abs{K_i}\geq \abs{A}$ and $\beta\geq 1$, part 3 above tells us that $K_i$ is the largest connected component in $C_i$.
\end{enumerate}
\end{proof}

\subsection{Proof of Lemma~\ref{lem:bisectorproperties}}
\label{app:sub:proof_of_bisector_properties}

We restate and prove Lemma~\ref{lem:bisectorproperties} below.
\begin{lemma*}
Suppose $\rho =\Omega(\Delta/\epsilon^2)$.
Then, for $a_i\in\nice{i,j}$ and $a_j\in\nice{j,i}$, we have
\begin{enumerate}
    \item $\norm{(a_i-a_j)_{(u)}} \geq \frac{\norm{(a_i-a_j)_{(V)}}}{\epsilon}$,
    \item $\inner{\frac{a_i+a_j}{2} - p, u}\leq \frac{\Delta}{2}$, and
    \item $\Bignorm{\left(\frac{a_i+a_j}{2} - p\right)_{(V)}} \leq \Delta/\epsilon$.
\end{enumerate}
\end{lemma*}

\begin{proof}
\leavevmode
\begin{enumerate}
\item We have $\norm{(a_i-a_j)_{(V)}}\leq 2\Delta/\epsilon$. On the other hand, $\rho \leq \norm{(a_i-a_j)_{(u)}}$. Thus the inequality holds for $\rho \geq 2\Delta/\epsilon^2$.
\item $\inner{a_i+a_j - 2p, u} = \inner{a_i-p,u}+\inner{a_j-p,u} \leq D_{i,j}/2 + (-D_{i,j}/2 + \Delta) = \Delta$. Multiplying by $1/2$ gives the desired inequality.
\item $\norm{(a_i+a_j - 2p)_{(V)}} \leq \norm{(a_i-p)_{(V)}} + \norm{(a_j-p)_{(V)}} \leq 2\Delta/\epsilon$. Multiplying by $1/2$ gives the desired inequality.
\end{enumerate}
\end{proof}


\section{Robust $k$-means}
\label{app:sec:robust}

For completeness, we restate Algorithm \ref{alg:robustKClustering} and Theorem \ref{thm:rho_epsilon_beta_eta_is_easy}.

\begin{algorithm*}
	\hrulefill
	\begin{algorithmic}[1]
		\Require $X\cup Z$, $r$, $t$
        \Procedure{\texttt{INITIALIZE}}{}
            \State Create graph $G$ on $X\cup Z$ where vertices $u$ and $v$ have an edge iff $\norm{u-v}<r$
            \State Remove vertices with vertex degree $<t$
            \State Let $a_1,\dots,a_k\in\R^d$ where  $a_i$ is the mean of the $i$th largest connected component of $G$
        \EndProcedure
        \Procedure{\texttt{ASSIGN}}{}
            \State Let $C_1,\dots,C_k$ be the clusters obtained by assigning each point in $I\cup Z$ to the closest $a_i$
        \EndProcedure
	\end{algorithmic}
	\hrulefill
\end{algorithm*}

\begin{theorem*}
Given $X\cup Z$ where $X$ satisfies $(\rho,\Delta,\epsilon)$-separation for
\[
\rho=\Omega\left(\frac{\Delta}{\epsilon^2}\left(\frac{w_{\max} + \eta}{w_{\min} - \eta}\right)\right),
\]
$\abs{X} = n$, and $\abs{Z}\leq \eta n$ for $\eta<w_{\min}$, there exists values of $r,t$ such that Algorithm \ref{alg:robustKClustering} returns a clustering consistent with $C_1,\dots,C_k$ on $X$. Algorithm \ref{alg:robustKClustering} can be implemented in $\widetilde O(n^2kd)$ time.
\end{theorem*}

Fix the following parameters.
\begin{align*}
\alpha = 2\left(\frac{w_{\max} + \eta}{w_{\min} - \eta} \right),
\hspace{2em}
r = \Delta(\alpha + 1)(1+2/\epsilon),
\hspace{2em}
t = w_{\min} n \frac{\alpha}{\alpha + 1}
.
\end{align*}

Define the following extended and robust versions of the regions defined in Section \ref{sec:kmeans}. Given $i,j$, let $C_i,C_j$ be the corresponding clusters with means $\mu_i,\mu_j$. Let $u=\frac{\mu_i-\mu_j}{\norm{\mu_i-\mu_j}}$ be the unit vector in the inter-mean direction.
\begin{definition}\leavevmode
	\begin{itemize}
    \item $\enice{i,j} = \set{x\in \cone{i,j} | \inner{x-\mu_i,u}\leq \alpha\Delta}$,
		\item $\renice{i,j} = \set{x\in\R^d| d(x, \enice{i,j})\leq r}$,
		\item $\rgood{i} = \bigcap_{j\neq i}\renice{i,j}$.
	\end{itemize}
\end{definition}

Again, it suffices to prove the following two lemmas. Lemma \ref{app:lem:robust_initialize_is_correct} states that the initialization returned by the \texttt{INITIALIZE} subroutine satisfies certain properties when given $r,t$.
As in the case of Algorithm \ref{alg:kClustering}, this algorithm uses $r$ and $t$ as thresholds. Hence, it is possible to guess $r$ from the $\binom{n}{2}$ pairwise edge lengths and $t$ from $[n]$ if necessary.
Lemma \ref{app:lem:robust_assign_is_correct} states that the \texttt{ASSIGN} subroutine correctly clusters all points given an initialization satisfying these properties.

\begin{lemma}
\label{app:lem:robust_initialize_is_correct}
Given $X\cup Z$ where $X$ is a $\APSM$ instance with $\rho=\Omega(\alpha\Delta/\epsilon^2)$ and $\eta < w_{\min}$, for the choices of $r$ and $t$ as above,
the \texttt{INITIALIZE} subroutine finds a set $\set{a_1,\dots,a_k}$ where $a_i\in\rgood{i}$
\end{lemma}

\begin{lemma}
\label{app:lem:robust_assign_is_correct}
Given $X\cup Z$ where $X$ is a $\APSM$ instance with $\rho=\Omega(\alpha\Delta/\epsilon^2)$ and $\eta < w_{\min}$,
the \texttt{ASSIGN} subroutine finds a clustering consistent with $C_1,\dots,C_k$ on $X$ when initialized with $k$ points $\set{a_1,\dots,a_k}$ where $a_i\in \rgood{i}$.
\end{lemma}

\subsection{Proof of Lemma \ref{app:lem:robust_initialize_is_correct}}

Consider the graph constructed by Algorithm \ref{alg:robustKClustering}.
The following lemma states some properties of the connected components in our graph.

\begin{lemma}
\label{app:lem:robust_conn_comp_properties}
\leavevmode
\begin{enumerate}
  \item For any $i\neq j$, the set of vertices $\enice{i,j}\cap X$ forms a clique and the size of this clique is greater than $t$. In particular, no vertex in $\enice{i,j}$ is deleted.
  \item Fix $i$. For all $j\neq i$, the vertices $\enice{i,j}\cap X$ belong to a single connected component. Let $K_i$ be this connected component.
  \item Before vertex deletion (and after), no vertex is adjacent to pure points from different clusters.
  \item After vertex deletion, every remaining point lies in $\rgood{i}$ for some $i$. Hence by part 2, every connected component contains pure points from at most a single cluster. In particular, $K_1,\dots,K_k$ are distinct.
\end{enumerate}
\end{lemma}
\begin{proof}
  \leavevmode
	\begin{enumerate}
		\item The diameter of $\enice{i,j}$ is $\diam(\enice{i,j}) \leq (\alpha+1)2\Delta/\epsilon<r$.
		Thus every pair of points in this region is connected.
		Recall that $\mu_i$ is the mean of the pure points in cluster $C_i$. By an averaging argument, $\abs{\enice{i,j}\cap X}\Delta - (\abs{C_i} - \abs{\enice{i,j}\cap X})\alpha\Delta \geq 0$. Rearranging, $\abs{\enice{i,j}\cap X} \geq \frac{\alpha}{\alpha + 1} \abs{C_i} \geq \frac{\alpha}{\alpha + 1}nw_{\min} = t$.

    \item Fix $i$. Let $j\neq i$. Recall $\nice{i,j}\cap X$ is nonempty; let $x\in\nice{i,j}\cap X$. Then $\norm{x-\mu_i}\leq \Delta/\epsilon$. We show that for any $j'\neq i$, the connected component containing $x$ contains $\enice{i,j'}\cap X$. Let $y\in \enice{i,j'}\cap X$. Then $\norm{y-x}\leq \norm{y-\mu_i}+\norm{x-\mu_i}\leq (\alpha + 1)\Delta/\epsilon + \alpha \Delta + \Delta/\epsilon < \Delta(\alpha+1)(1+2/\epsilon)=r$.

    \item Pure points in different clusters are at distance at least $\rho$ whereas two vertices sharing a neighbor must be at distance less than $2r$. Thus the inequality holds for $\rho \geq \Omega(\alpha\Delta/\epsilon)$.

    \item Let $x$ be a point not in $\bigcup_i\rgood{i}$. By part 3 above, $x$ can only be connected to pure points in a single cluster. Suppose it is connected to pure points in cluster $C_i$. By assumption, there exists a $j$ such that $x\notin \renice{i,j}$. We bound the degree of $x$ above by the number of points in $X\setminus\enice{i,j}$ and the $\eta n$-many impure points, i.e., $\deg(x)\leq \eta n + \frac{\abs{C_i}}{\alpha +1} \leq n(\eta + \frac{w_{\max}}{\alpha + 1})$. By our choice of $t$, we have that $\deg(x)<t$. Thus $x$ is deleted and all remaining points lie in $\bigcup_i \rgood{i}$.

		For any $i,j$, the minimum distance between $\rgood{i}$ and $\rgood{j}$ is at least $\rho-2r$. For some $\rho\geq\Omega(\alpha\Delta/\epsilon)$ then, the distance between these regions is greater than $\rho - 2r > r$ and no connected component contains pure points from multiple clusters.
	\end{enumerate}
\end{proof}

Lemma \ref{app:lem:robust_largest_in_diff} state that the $k$ largest components contain pure points corresponding to different clusters while Lemma \ref{app:lem:robust_mean_in_good} states that each $a_i$ lies inside a robust good region. Together, they imply Lemma \ref{app:lem:robust_initialize_is_correct}, i.e. each $a_i$ lies in a different robust good region.

\begin{lemma}
\label{app:lem:robust_largest_in_diff}
Let $K_i$ be defined as above. For any arbitrary connected component $K$ not in $K_1,\dots, K_k$, $\abs{K_i}>\abs{K}$. In particular, the $k$ largest components of $G$ are $K_1,\dots,K_k$.
\end{lemma}
\begin{proof}
As in part 2 above, the size of $K_i$ is bounded below by the averaging argument $\abs{K_i}\geq \frac{\alpha}{\alpha+1}\abs{C_i}$. By part 3 above, $K$ contains pure points from at most a single cluster $C_j$. By part 5 above, the size of the connected component $K$ is bounded above by the number of remaining points after $K_j$ is removed and the $\eta n$-many impure points, i.e., $\abs{C_j} \leq \frac{1}{\alpha+1}\abs{C_j} + \eta n$. Then by our choice of $\alpha$, $\abs{K}< \abs{K_i}$.
\end{proof}

\begin{lemma}
\label{app:lem:robust_mean_in_good}
The mean of $K_i$ lies in $\rgood{i}$.
\end{lemma}
\begin{proof}
By above, $K_i\subseteq \rgood{i}$. As $\rgood{i}$ is convex, the mean of $K_i$ also lies in $\rgood{i}$.
\end{proof}

\subsection{Proof of Lemma \ref{app:lem:robust_assign_is_correct}}

We will show that for any $a_i\in\renice{i,j}$, $a_j\in\renice{j,i}$ and $x\in C_i$, $x$ is closer to $a_i$ than $a_j$.
The following lemma states some properties of the perpendicular bisector between $a_i$ and $a_j$.

\begin{lemma}
	\label{app:lem:robustbisectorproperties}
	Suppose $\rho= \Omega(\alpha\Delta/\epsilon^2)$. Then, for $a_i\in\renice{i,j}$ and $a_j\in\renice{j,i}$, we have
	\begin{enumerate}
		\item $\norm{(a_i-a_j)_{(u)}} \geq \frac{\norm{(a_i-a_j)_{(V)}}}{\epsilon}$,
		\item $\inner{\frac{a_i+ a_j}{2} - p, u}\leq (\alpha+1)\Delta/2 + r$,
		\item $\Bignorm{\left(\frac{a_i+a_j}{2} - p\right)_{(V)}}\leq (\alpha+1)\Delta/\epsilon+r$.
	\end{enumerate}
\end{lemma}
\begin{proof}
\leavevmode
\begin{enumerate}
	\item By triangle inequality, $\norm{(a_i-a_j)_{(V)}}\leq 2((\alpha + 1)\Delta/\epsilon + r)$. On the other hand, $\norm{(a_i-a_j)_{(u)}} \geq \rho - 2r$. Thus the inequality holds for $\rho \geq 2r + \frac{2}{\epsilon}((\alpha+1)\Delta/\epsilon +r)$.
	\item $\inner{a_i+a_j - 2p,u} = \inner{a_i-p,u} + \inner{a_j-p,u} \leq (D_{i,j}/2 +\alpha\Delta + r) + (-D_{i,j}/2 + \Delta + r) = (\alpha + 1)\Delta + 2r$. Multiplying by $1/2$ gives the desired inequality.
	\item $\norm{(a_i+a_j-2p)_{(V)}} \leq\norm{(a_i - p)_{(V)}}+\norm{(a_j - p)_{(V)}} \leq 2 ((\alpha + 1)\Delta/\epsilon + r)$. Multiplying by $1/2$ gives the desired inequality.
\end{enumerate}
\end{proof}

To prove Lemma \ref{app:lem:robust_assign_is_correct}, we rewrite the condition $\norm{x-a_i} \leq \norm{x-a_j}$ as $\inner{(x-p)-(\tfrac{1}{2}(a_i+a_j)-p),a_i-a_j} \geq 0$. Then we write each vector in terms of their projection on $u$ and $V$ and use the above lemma to bound each of the terms.

\begin{proof}[Proof of Lemma \ref{app:lem:robust_assign_is_correct}]
	It suffices to show that for any $a_i\in\renice{i,j}$, $a_j\in\renice{j,i}$ and $x\in C_i$, $\norm{x-a_i}\leq \norm{x-a_j}$.
	Then by Lemma \ref{app:lem:robustbisectorproperties} above,
	\begin{align*}
	\Iprod{(x-p) - \left(\frac{a_i+a_j}{2}-p\right), a_i-a_j} &= \Iprod{(x-p)_{(u)}, (a_i-a_j)_{(u)}} + \Iprod{(x-p)_{(V)}, (a_i-a_j)_{(V)}}\\
	&\phantom{=}\hspace{2em}- \tfrac12\Iprod{(a_i + a_j-2p)_{(u)}, (a_i-a_j)_{(u)}}\\
	&\phantom{=}\hspace{2em} - \tfrac12\Iprod{(a_i + a_j-2p)_{(V)}, (a_i-a_j)_{(V)}}\\
	&\geq \norm{(x-p)_{(u)}}\norm{(a_i-a_j)_{(u)}} - \frac{1}{\epsilon}\left(\norm{(x-p)_{(u)}} - \rho/2\right)\epsilon\norm{(a_i-a_j)_{(u)}}\\
	&\phantom{>}\hspace{2 em} - \left((\alpha+1)\Delta/2 + r\right)\norm{(a_i-a_j)_{(u)}}\\
	&\phantom{>}\hspace{2em} - \left((\alpha + 1)\Delta/\epsilon + r\right)\epsilon \norm{(a_i-a_j)_{(u)}}\\
	&=\left(\frac{\rho}{2} - \left( \frac{3}{2}(\alpha+1)\Delta + (1+\epsilon)r \right)\right)\norm{(a_i-a_j)_{(u)}}
	\end{align*}
	where the inequality follows because of equality on the first term and Cauchy-Schwarz on the rest. So, when $\rho=\Omega(\alpha\Delta/\epsilon^2)$, for all $a_i\in \renice{i,j}$, $a_j\in\renice{j,i}$, and $x\in C_i$, $x$ is closer to $a_i$ than $a_j$.
\end{proof}
\end{appendices}

\end{document}